\documentclass{article}

\usepackage{microtype}
\usepackage{graphicx}
\usepackage{subcaption}
\usepackage{booktabs} 

\usepackage{hyperref}

\usepackage[preprint]{icml2026}

\usepackage{amsmath}
\usepackage{amssymb}
\usepackage{mathtools}
\usepackage{amsthm}

\usepackage[capitalize,noabbrev]{cleveref}

\theoremstyle{plain}
\newtheorem{theorem}{Theorem}[section]

\newtheorem{lemma}[theorem]{Lemma}

\theoremstyle{definition}

\newtheorem{assumption}[theorem]{Assumption}
\theoremstyle{remark}

\usepackage{algorithm}
\usepackage{algorithmic}
\usepackage{color}
\usepackage{amsfonts}
\newcommand{\bs}{\boldsymbol}
\newcommand{\mc}{\mathcal}
\renewcommand{\rm}{\mathrm}

\usepackage{comment}
\usepackage{bm}
\usepackage{nicefrac}

\usepackage[textsize=tiny]{todonotes}

\icmltitlerunning{Linear Convergence in Games with Delayed Feedback via Extra Prediction}

\begin{document}

\twocolumn[
  \icmltitle{Linear Convergence in Games with Delayed Feedback via Extra Prediction}



  \icmlsetsymbol{equal}{*}

  \begin{icmlauthorlist}
    \icmlauthor{Yuma Fujimoto}{ca,soken}
    \icmlauthor{Kenshi Abe}{ca}
    \icmlauthor{Kaito Ariu}{ca}
  \end{icmlauthorlist}

  \icmlaffiliation{ca}{CyberAgent, Tokyo, Japan}
  \icmlaffiliation{soken}{Soken University, Kanagawa, Japan}

  \icmlcorrespondingauthor{Yuma Fujimoto}{fujimoto.yuma1991@gmail.com}

  \icmlkeywords{Bilinear Games, Learning in Games, Last-Iterate Convergence, Optimistic}

  \vskip 0.3in
]



\printAffiliationsAndNotice{}  

\begin{abstract}
Feedback delays are inevitable in real-world multi-agent learning. They are known to severely degrade performance, and the convergence rate under delayed feedback is still unclear, even for bilinear games. This paper derives the rate of linear convergence of Weighted Optimistic Gradient Descent-Ascent (WOGDA), which predicts future rewards with extra optimism, in unconstrained bilinear games. To analyze the algorithm, we interpret it as an approximation of the Extra Proximal Point (EPP), which is updated based on farther future rewards than the classical Proximal Point (PP). Our theorems show that standard optimism (predicting the next-step reward) achieves linear convergence to the equilibrium at a rate $\exp(-\Theta(t/m^{5}))$ after $t$ iterations for delay $m$. Moreover, employing extra optimism (predicting farther future reward) tolerates a larger step size and significantly accelerates the rate to $\exp(-\Theta(t/(m^{2}\log m)))$. Our experiments also show accelerated convergence driven by the extra optimism and are qualitatively consistent with our theorems. In summary, this paper validates that extra optimism is a promising countermeasure against performance degradation caused by feedback delays.
\end{abstract}

\section{Introduction}
Online learning aims for efficient sequential decision-making. Typically, it assumes an ideal situation in which current strategies can be determined from all the past feedback. In real-world online learning scenarios, however, delays in feedback are generally inevitable. For instance, in online advertising, there is often a significant time lag between displaying an ad and observing a conversion~\cite{chapelle2014modeling, yoshikawa2018nonparametric, yasui2020feedback}. Similarly, in distributed learning, communication latency and asynchronous updates inherently introduce delays in gradient aggregation~\cite{agarwal2011distributed, mcmahan2014delay, zheng2017asynchronous}. Indeed, a considerable number of papers on online learning are motivated by such feedback delays and report that delays amplify regret for full feedback~\cite{weinberger2002delayed, zinkevich2009slow, quanrud2015online, joulani2016delay, shamir2017online} and bandit feedback~\cite{neu2010online, joulani2013online, desautels2014parallelizing, cesa2016delay, vernade2017stochastic, pike2018bandits, cesa2018nonstochastic, li2019bandit}.

Such feedback delays have also been of interest in multi-agent learning or learning in games~\cite{zhou2017countering, hsieh2022multi} and are known to severely degrade performance~\cite{fujimoto2025learning}. This is because good performance in multi-agent learning is based on each agent predicting their future reward, and feedback delays make this prediction more challenging. Indeed, Optimistic Follow the Regularized Leader (OFTRL), which is a predictive algorithm and enjoys $O(1)$-regret under instantaneous feedback, suffers from $\Omega(\sqrt{T})$-regret for the time horizon $T$. Even if we adopt a delay-correction mechanism called ``Weighted'' OFTRL (WOFTRL), the regret scales as $O(m^{2})$, growing too large with delay $m$.

Despite these prior studies, fundamental challenges remain in delayed feedback in multi-agent learning, especially about convergence analysis in bilinear games, defined as
\begin{align}
    \max_{\bs{x}\in\mc{X}}\min_{\bs{y}\in\mc{Y}}\bs{x}^{\rm T}\bs{B}\bs{y}. \tag{bilinear game} \label{bilinear}
\end{align}
The prior study~\cite{fujimoto2025learning} proved WOFTRL converges to the equilibrium, called last-iterate convergence (LIC), when $\mc{X}$ and $\mc{Y}$ are constrained to probability spaces. However, LIC in the unconstrained setting $\mc{X}=\mathbb{R}^{d_{\rm{X}}}$ and $\mc{Y}=\mathbb{R}^{d_{\rm{Y}}}$ under delayed feedback is not guaranteed. Furthermore, its convergence rate is still unestablished. Finding this rate is vital for understanding how quickly agents can stabilize their strategies in applications, and thus convergence rate is an attractive topic in the context of learning in games. Lastly, although the experiment in the prior research suggests that predicting the farther future than necessary to correcting the delays (called ``extra prediction'') results in faster convergence, the validity of this extra prediction has yet to be established.

In this paper, we address these open problems in unconstrained bilinear games. Our contributions are as follows.
\begin{itemize}
\item {\bf We establish the rate of linear convergence even with feedback delays.} We approximate our algorithm WOGDA by Extra Proximal Point (EPP), an extension of classical Proximal Point (PP) to predict future rewards. We prove that EPP linearly converges and that the difference between WOGDA and EPP is sufficiently small by setting the step size appropriately.
\item {\bf We demonstrate that extra prediction accelerates convergence.} We find that extra prediction permits larger step sizes, and the underlying EPP converges faster. Therefore, WOGDA with extra prediction achieves much faster convergence at the scale of delay $m$.
\item {\bf Our theoretical results are also reproduced in experiments.} Both the linear convergence and the acceleration by extra prediction are observed in experiments using both representative (Matching Pennies) and unintended ($5\times 5$ random matrix) games.
\end{itemize}

\paragraph{Unconstrained bilinear games:} This study targets the class of unconstrained bilinear games. This class is closely related to min-max optimization, and convergence is an important issue there. Also, unconstrained bilinear games are one of the minimum necessary configurations (zero-sum utility and Euclidean strategy space), including difficulties specific to multi-agent learning, and thus have the potential to develop into other various advanced configurations, such as convex-concave utility and constrained strategy space. Indeed, the celebrated study showing LIC in unconstrained bilinear games~\cite{daskalakis2018training} has been thereafter applied to the constrained setting~\cite{mertikopoulos2019optimistic}. The linear convergence was first demonstrated in unconstrained bilinear games~\cite{mokhtari2020unified} and later shown to hold in constrained saddle-point problems~\cite{wei2021linear}. Also, LIC in time-varying games was first proven in unconstrained bilinear games~\cite{feng2023last} and later discussed for the constrained setting~\cite{feng2024last, fujimoto2025synchronization}. To summarize, unconstrained bilinear games serve as a touchstone for analyzing novel phenomena in learning in games.

\paragraph{Analysis Based on PP method:} 
Many algorithms in learning in games are analyzed based on the PP method. Here, PP linearly converges to equilibrium and thus is a powerful algorithm, but its applicability is limited because it uses the information on the next-step reward (thus, implicit and coupled method). For example, in the unconstrained setting, PP is used to evaluate predictive variant of gradient descent-ascent (GDA), i.e., optimistic GDA (OGDA)~\cite{mokhtari2020unified}. In the constrained setting, optimistic algorithms are analyzed using the next-step reward~\cite{rakhlin2013optimization, syrgkanis2015fast}. A method that approximates this PP with arbitrary precision has been proposed~\cite{piliouras2022beyond, cevher2023min}. Overall, the advantage of predicting the next step has been widely discussed, but this advantage is still unexplained for predicting farther future than the next step.

\section{Unconstrained Bilinear Game}
This study addresses \ref{bilinear} with unconstrained setting $\mc{X}=\mathbb{R}^{d_{\rm X}}$ and $\mc{Y}=\mathbb{R}^{d_{\rm Y}}$. The solution of this bilinear game is given by $(\bs{x}_{*},\bs{y}_{*})$ such that
\begin{align}
    \bs{B}^{\rm T}\bs{x}_{*}=\bs{0},\quad \bs{B}\bs{y}_{*}=\bs{0}.
\end{align}

For each time $t\in\{1,\cdots, T\}$, each player determines their strategies as $\bs{x}_{t}\in \mathbb{R}^{d_{\rm X}}$ and $\bs{y}_{t}\in \mathbb{R}^{d_{\rm Y}}$. Each game play returns the full feedback, where each player observes the gradient of their objective function, i.e., $\bs{u}_{t}=\bs{B}\bs{y}_{t}$ for X and $\bs{v}_{t}=-\bs{B}^{\rm T}\bs{x}_{t}$ for Y.

\subsection{Delayed Feedback}
While standard online learning assumes instantaneous feedback where each player can use all their strategies and receive rewards up to time $t$ to determine their next strategy, this study addresses a realistic scenario involving a fixed time delay of $m$ steps in observing their rewards. Formally, although player X should determine their next strategy as a function $\bs{x}_{t+1}(\{\bs{x}_{s}\}_{0\le s\le t}, \{\bs{u}_{s}\}_{0\le s\le t})$ in the standard setting, suppose that under delayed feedback with $m\in\mathbb{N}$, they should instead determine a function
\begin{align}
    \bs{x}_{t+1}(\{\bs{x}_{s}\}_{0\le s\le t}, \{\bs{u}_{s}\}_{0\le s\le t-m}). \tag{Delayed Feedback}
\end{align}

\subsection{Notation and Assumption}
For convenience, we introduce the notations of
\begin{align}
    &\bs{z}_{t}=\begin{pmatrix}
        \bs{x}_{t} \\
        \bs{y}_{t} \\
    \end{pmatrix},\quad \bs{w}_{t}=\begin{pmatrix}
        \bs{u}_{t} \\
        \bs{v}_{t} \\
    \end{pmatrix}, \label{concate_vector}\\
    &\bs{A}=\begin{pmatrix}
        \bs{O} & \bs{B} \\
        -\bs{B}^{\rm T} & \bs{O} \\
    \end{pmatrix},\quad \bar{\bs{A}}=\begin{pmatrix}
        \bs{O} & \bs{B} \\
        \bs{B}^{\rm T} & \bs{O} \\
    \end{pmatrix}. \label{concate_matrix}
\end{align}
Here, $\bs{z}_{t}$ and $\bs{w}_{t}$ are the concatenation of the strategies and rewards, respectively. $\bs{A}$ is the matrix such that $\bs{w}_{t}=\bs{A}\bs{z}_{t}$ holds. We also defined $\bar{\bs{A}}$, which satisfies $\bar{\bs{A}}^{\rm T}=\bar{\bs{A}}$ and $\bs{A}^{\rm T}\bs{A}=\bar{\bs{A}}^{2}$.
\begin{align}
    \bs{A}^{\rm T}=-\bs{A},\quad \bs{A}^{\rm T}\bs{A}=-\bs{A}^{2}=\bar{\bs{A}}^{2}.
\end{align}

Following the previous studies on linear convergence to the Nash equilibrium~\cite{mokhtari2020unified}, we also assume that $\bs{B}$ is a regular matrix.

\begin{assumption}[Regular Matrix Game]
Suppose that $d_{\rm X}=d_{\rm Y}=:d$ and $\bs{B}$ is a regular matrix.
\end{assumption}

This assumption is convenient for discussing several convergence properties. First, because $\bs{B}$ has the inverse matrix, the Nash equilibrium is uniquely given by $\bs{x}_{*}=\bs{y}_{*}=\bs{0}$. Second, because $\bs{B}$ is a full-rank matrix, the skew-symmetric matrix $\bs{A}$ has only pure imaginary eigenvalues, denoted as $\pm i\lambda_{\min},\cdots,\pm i\lambda_{\max}$ for some $(0<)\lambda_{\min}\le\cdots\le\lambda_{\max}$. Thus, we can evaluate
\begin{align}
    0<\lambda_{\min}\|\bs{z}\|\le \|\bs{A}\bs{z}\|\le \lambda_{\max}\|\bs{z}\|,
\end{align}
for all $\bs{z}\in\mathbb{R}^{2d}\setminus\{\bs{0}\}$. We note that we use $\|\cdot\|$ indicates the Euclidean norm $\|\cdot\|_{2}$ throughout this paper. We also define the ratio of the maximum and minimum eigenvalues as
\begin{align}
    \kappa=\frac{\lambda_{\max}}{\lambda_{\min}}(>1).
\end{align}

\paragraph{Remark:} Bilinear games can be extended to the following alternative problem
\begin{align}
    \max_{\bs{x}\in\mathbb{R}^{d_{\rm X}}}\min_{\bs{y}\in\mathbb{R}^{d_{\rm Y}}}\bs{x}^{\rm T}\bs{B}\bs{y}+\bs{x}^{\rm T}\bs{c}'+\bs{c}^{\rm T}\bs{y},
\end{align}
where $\bs{c}, \bs{c}'\in\mathbb{R}^{d}$ are unobservable vectors. The equilibrium of this alternative problem is $\bs{x}_{*}=-(\bs{B}^{\rm T})^{-1}\bs{c}$ and $\bs{y}_{*}=-\bs{B}^{-1}\bs{c}'$. As similar to the study~\cite{daskalakis2018training}, our algorithm is applicable to this alternative problem. Even though one might consider that $\bs{x}_{*}=\bs{y}_{*}=\bs{0}$ is trivial equilibrium, this study, just for the simplicity of notation, chooses the special case of $\bs{c}=\bs{c}'=\bs{0}$ without loss of generality.

\section{Algorithm}
In this work, we focus on Weighted Optimistic Gradient Descent-Ascent (WOGDA), which is the application of WOFTRL and WOMD algorithms~\cite{fujimoto2025learning} to unconstrained bilinear games. (Technically, WOGDA is obtained by replacing the constrained strategy domain with a Euclidean space and employing a Euclidean regularizer in WOFTRL and WOMD.)

\subsection{WOGDA}
In the delayed feedback setting, the next strategy $\bs{z}_{t+1}$ should be determined by the delayed rewards, i.e., $\bs{w}_{1},\cdots,\bs{w}_{t-m}$. With extra prediction length $n\in\mathbb{N}$, the WOGDA gives $\bs{z}_{t+1}$ as follows.
\begin{align}
    \begin{cases}
    \hat{\bs{z}}_{t+1}=\hat{\bs{z}}_{t}+\eta\bs{w}_{t+1} \\
    \bs{z}_{t+1}=\hat{\bs{z}}_{t-m}+(n+m)\eta\bs{w}_{t-m}
    \end{cases}.\tag{WOGDA} \label{WOGDA}
\end{align}
Because the next strategy $\bs{z}_{t+1}$ depends on $\hat{\bs{z}}_{t-m}$ (the first term) determined only by $\bs{w}_{1},\cdots,\bs{w}_{t-m}$ and $\bs{w}_{t-m}$ (second), WOGDA is executable even with feedback delays $m$. By using the following notation
\begin{align}
    \Delta\bs{a}_{t}=\bs{a}_{t+1}-\bs{a}_{t}, \tag{$\Delta$-notation} \label{delta_notation}
\end{align}
\ref{WOGDA} is also written as
\begin{align}
    \Delta\bs{z}_{t}&=\eta\bs{A}\bs{z}_{t-m}+(n+m)\eta\bs{A}\Delta\bs{z}_{t-m-1}. \label{WOGDA_delta}
\end{align}

\paragraph{Connection to GDA and OGDA:} If there is no delay ($m=0$), the WOGDA immediately reproduces the GDA ($n=0$) and OGDA ($n=1$) as
\begin{align}
    \bs{z}_{t+1}=&\ \bs{z}_{t}+\eta\bs{A}\bs{z}_{t}, \tag{GDA} \\
    \bs{z}_{t+1}=&\ \bs{z}_{t}+\eta\bs{A}\bs{z}_{t}+\eta\bs{A}(\bs{z}_{t}-\bs{z}_{t-1}). \tag{OGDA} \label{OGDA}
\end{align}

\paragraph{Interpretation:} 
Here, $\hat{\bs{z}}_{t}$ is a dummy variable and represents the cumulative rewards until the current time $t$. Thus, the strategy $\bs{z}_{t+1}$ is given by the observable cumulative rewards $\hat{\bs{z}}_{t-m}$ and weighting the latest rewards $(n+m)$ times. Here, the weight $m$ is interpreted as canceling out the delay, because depending on the delay length $m$, the number of observable rewards is reduced by $m$ in $\hat{\bs{z}}_{t-m}$. On the other hand, the weight $n$ is interpreted as predicting the future reward, because the cumulative of future rewards $\bs{w}_{t+1},\cdots,\bs{w}_{t+n}$ is extrapolated by the latest reward $\bs{w}_{t-m}$. Therefore, when $n=1$, WOGDA predicts the next-step cumulative reward, called ``next-step prediction''. When $n>1$, it predicts further future cumulative rewards, called ``extra prediction''.

\section{Linear Convergence of WOGDA}
This section presents the main theorems, showing the rate of linear convergence of WOGDA. We consider two cases: $n=1$ (next-step prediction) and $n=m/2+1$ (extra prediction). Interestingly, depending on the difference of these $n$, the possible step size and convergence rate change on the scale of $m$. We skip the proofs of these main theorems until the later section, and see Appendix for the omitted full proofs.

\subsection{Convergence by Next-Step Prediction}
Following the conventional methodology, suppose $n=1$, where WOGDA predicts the reward at the next step, then we can provide the rate of linear convergence as follows.

\begin{theorem}[Linear Convergence by Next-Step Prediction] \label{thm_convergence_pp}
Suppose $n=1$ and $\eta=1/(56(m+1)^{2}\kappa^{2}\lambda_{\max})$, then it holds
\begin{align}
    \|\bs{z}_{t}\|&\le \exp\left(-\frac{t}{c\kappa^{6}(m+1)^{5}}\right)\|\tilde{\bs{z}}_{0}\|, \label{rate_pp}
\end{align}
with some positive constant $c(>0)$. Here, $\|\tilde{\bs{z}}_{0}\|$ is defined as
\begin{align}
    \|\tilde{\bs{z}}_{0}\|&=\max_{0\le s\le 4(m+1)}\|\bs{z}_{s}\|.
\end{align}
\end{theorem}

\paragraph{Connection to Prior Research:} Theorem \ref{thm_convergence_pp} establishes a linear convergence rate, guaranteeing exponential convergence to equilibrium, whereas this remained an open question in previous work~\cite{fujimoto2025learning}. Although our setting (unconstrained bilinear) differs to some extent from theirs (poly-matrix zero-sum), both approaches share the underlying mechanism (approximating the PP method) and the step size scaling ($\eta=O(1/m^{2})$).

\subsection{Convergence by Extra Prediction}
Beyond the conventional methodology, suppose $n=m/2+1$, where WOGDA extra-predicts the reward beyond the next step, then we can prove accelerated linear convergence as follows.

\begin{theorem}[Linear Convergence by Extra Prediction] \label{thm_convergence_epp}
Suppose $n=m/2+1$ and $\eta=1/(93(m+1)^{\frac{2j}{2j-1}}\kappa^{2}\lambda_{\max})$ with $j=\lfloor\log(m+1)\rfloor+2$, then it holds
\begin{align}
    \|\bs{z}_{t}\|&\le \exp\left(-\frac{t}{c\kappa^{6}(m+1)^2(\log(m+1)+2)}\right)\|\tilde{\bs{z}}_{0}\|, \label{rate_epp}
\end{align}
with some positive constant $c(>0)$. Here, $\|\tilde{\bs{z}}_{0}\|$ is defined as
\begin{align}
    \|\tilde{\bs{z}}_{0}\|&=\max_{0\le s\le 2(m+1)(\log(m+1)+2)}\|\bs{z}_{s}\|.
\end{align}
\end{theorem}

\paragraph{Accelerated convergence by extra prediction:} Comparing Theorem \ref{thm_convergence_pp} with Theorem \ref{thm_convergence_epp}, we see that the convergence rate is quite accelerated from $\exp(-\Theta(t/m^{5}))$ in next-step prediction to $\exp(-\Theta(t/(m^{2}\log m)))$ in extra prediction. We also see that the step size increases from $\eta=O(1/m^{2})$ in next-step prediction to $\eta=O(1/(m\log m))$ in extra prediction. Since a larger step size is known to destabilize convergence in general, extra prediction enhances the stability of WOGDA. In conclusion, extra prediction is a strong countermeasure against feedback delays in both accelerated and stabilized convergence, as later supported by rigorous proofs and experiments.

\section{Proof of Theorems \ref{thm_convergence_pp} and \ref{thm_convergence_epp}}
To prove Theorems \ref{thm_convergence_pp} and \ref{thm_convergence_epp}, we analyze the dynamics of \ref{WOGDA} for all prediction lengths $n\in\mathbb{N}$ and step sizes $0\le \eta\le O(1/(n+m))$. The proof proceeds as follows. In \S\ref{ss_epp}, we prove that EPP, an implicit method approximating WOGDA, linearly converges. In \S\ref{ss_error}, we evaluate the approximation error between WOGDA and EPP. In \S\ref{ss_combination}, we combine the convergence term by EPP and the error term.

\subsection{Analysis of EPP} \label{ss_epp}
We introduce EPP, an extension of the classical PP method. For any $\bs{z}_{t}$, its EPP $\bs{z}^{\dagger}_{t+1}$ is determined as
\begin{align}
    \bs{z}^{\dagger}_{t+1}=&\ \bs{z}_{t}+\eta\bs{A}\bs{z}_{t}+n\eta\bs{A}(\bs{z}^{\dagger}_{t+1}-\bs{z}_{t}). \tag{EPP} \label{EPP}
\end{align}
Here, $n\in\mathbb{R}$ means the prediction length, corresponding to that in \ref{WOGDA}.

\paragraph{Connection to GDA and PP:} This EPP linearly connects well-known algorithms, GDA ($n=0$) and PP ($n=1$) as
\begin{align}
    \bs{z}^{\dagger}_{t+1}=&\ \bs{z}_{t}+\eta\bs{A}\bs{z}_{t}, \tag{GDA} \\
    \bs{z}^{\dagger}_{t+1}=&\ \bs{z}_{t}+\eta\bs{A}\bs{z}^{\dagger}_{t+1}. \tag{PP} \label{PP}
\end{align}
When $n=0$, EPP corresponds to GDA and updates its strategy based on the previous reward $\bs{A}\bs{z}_{t}$. When $n=1$, EPP corresponds to PP and is based on the next-step reward $\bs{A}\bs{z}^{\dagger}_{t+1}$. When $n>1$, EPP is based on the $n$-step future reward, meaning extra prediction.

The dynamics of \ref{EPP} satisfy the following lemma.

\begin{lemma}[Dynamics of EPP] \label{lem_dynamics_epp}
For all $n\in\mathbb{R}$, $0\le\eta$, and $\bs{z}_{t}\in\mathbb{R}^{2d}$, it holds
\begin{align}
    \|\bs{z}^{\dagger}_{t+1}\|^{2}=\bs{z}_{t}^{\rm T}\frac{\bs{I}+(n-1)^{2}\eta^{2}\bar{\bs{A}}^{2}}{\bs{I}+n^{2}\eta^{2}\bar{\bs{A}}^{2}}\bs{z}_{t}. \label{EPP_equality}
\end{align}
\end{lemma}

\paragraph{Fraction of matrices:} Here, note that we abused the fraction notation of matrices for convenience. This notation has no mathematical contradiction because the inverse matrix of the denominator ($\bs{I}+n^{2}\eta^{2}\bar{\bs{A}}^{2}$) exists, and it is also commutative with the numerator ($\bs{I}+(n-1)^{2}\eta^{2}\bar{\bs{A}}^{2}$). We will henceforth use this notation throughout this paper.

{\it Proof Sketch.}
First, \ref{EPP} can be rewritten as
\begin{align}
    (\bs{I}-n\eta\bs{A})\bs{z}^{\dagger}_{t+1}=(\bs{I}-(n-1)\eta\bs{A})\bs{z}_{t}.
\end{align}
Since $\bs{A}$ is skew-symmetric, the inverse matrix of $I-n\eta\bs{A}$ exists and is also commutative with $\bs{I}-(n-1)\eta\bs{A}$. Therefore, we obtain
\begin{align}
    \bs{z}^{\dagger}_{t+1}=\frac{\bs{I}-(n-1)\eta\bs{A}}{I-n\eta\bs{A}}\bs{z}_{t}.
\end{align}
which derives Eq.~\eqref{EPP_equality}. Here, note that by utilizing the properties of Eqs.~\eqref{concate_matrix}, the cross terms between $\bs{x}_{t}$ and $\bs{y}_{t}$ cancels out and only the squared terms (corresponding to $\bar{\bs{A}}^{2}$) remain.
\qed

Since $\bar{\bs{A}}$ has only positive eigenvalues, Lemma \ref{lem_dynamics_epp} immediately implies linear convergence as follows.

\begin{lemma}[Analysis of EPP] \label{lem_analysis_epp}
For all $n\in\mathbb{R}$, $0\le\eta\le 1/(n\lambda_{\min})$, and $\bs{z}_{t}\in\mathbb{R}^{2d}$, it holds
\begin{align}
    \|\bs{z}^{\dagger}_{t+1}\|&\le\rm{LCR}_{\rm{EPP}}(n)\|\bs{z}_{t}\|, \\
    \rm{LCR}_{\rm{EPP}}(n)&=1-\frac{2n-1}{2}\eta^{2}\lambda_{\min}^{2}+\frac{n^{2}(2n-1)}{2}\eta^{4}\lambda_{\min}^{4}. \tag{Linear Convergence Rate of EPP}
\end{align}
\end{lemma}

\begin{proof}
By Lemma \ref{lem_dynamics_epp}, we obtain
\begin{align}
    \|\bs{z}^{\dagger}_{t+1}\|^{2}&=\bs{z}_{t}^{\rm T}\frac{\bs{I}+(n-1)^{2}\eta^{2}\bar{\bs{A}}^{2}}{\bs{I}+n^{2}\eta^{2}\bar{\bs{A}}^{2}}\bs{z}_{t} \nonumber\\
    &=\frac{1}{n^{2}}\bs{z}_{t}^{\rm T}\left((n-1)^{2}\bs{I}+\frac{(n^{2}-(n-1)^{2})\bs{I}}{\bs{I}+n^{2}\eta^{2}\bar{\bs{A}}^{2}}\right)\bs{z}_{t} \nonumber\\
    &\le\frac{1}{n^{2}}\left((n-1)^{2}+\frac{n^{2}-(n-1)^{2}}{1+n^{2}\eta^{2}\lambda_{\min}^{2}}\right)\|\bs{z}_{t}\|^{2} \nonumber\\
    &=\frac{1+(n-1)^{2}\eta^{2}\lambda_{\min}^{2}}{1+n^{2}\eta^{2}\lambda_{\min}^{2}}\|\bs{z}_{t}\|^{2} \label{epp_tight} \\
    &\le \rm{LCR}_{\rm{EPP}}(n)^{2}\|\bs{z}_{t}\|^{2}.
\end{align}
In the final inequality, we used that the inequality (holding for all $0\le a\le b\le 1$)
\begin{align}
    \frac{1+a}{1+b}& \le 1-(b-a)(1-b) \nonumber\\
    & \le\left(1-\frac{1}{2}(b-a)(1-b)\right)^{2}.
\end{align}
for $a=(n-1)^{2}\eta^{2}\lambda_{\min}^{2}$ and $b=n^{2}\eta^{2}\lambda_{\min}^{2}$.
\end{proof}

\paragraph{Interpretation:} Lemma \ref{lem_analysis_epp} is interpreted as $\rm{LCR}_{\rm{EPP}}(n)=1-O(n\eta^{2})$, meaning that EPP achieves linear convergence under sufficiently small step sizes. Furthermore, it also shows that a larger $n$ leads to faster convergence. Thus, if some external factor --- just like delayed feedback --- requires a small step size, EPP with extra prediction ($n>1$) outperforms that with next-step prediction ($n=1$, corresponding to classical PP). We note that if no such factor exists, EPP no longer improves PP. Indeed, the tighter rate, i.e., Eq.~\eqref{epp_tight}, shows that PP ($n=1$) converges arbitrarily fast in the limit of $\eta\to\infty$.

\begin{figure*}[h!]
    \centering
    \includegraphics[width=0.6\hsize]{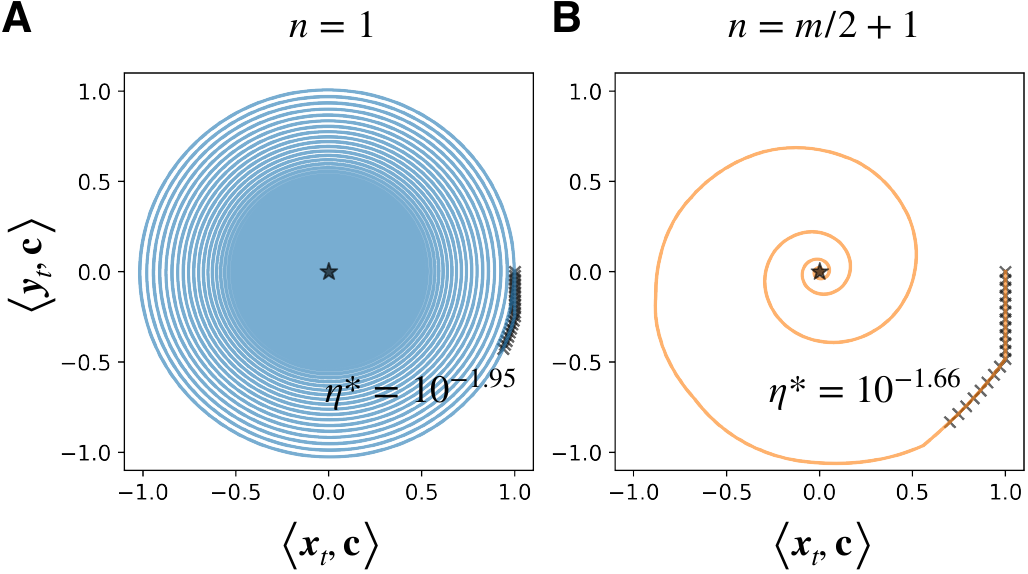}
    \caption{Convergence in Matching Pennies with delayed feedback. We set the delay as $m=10$ and the initial state as $(\bs{x}_{0},\bs{y}_{0})=({\bf c}/2,\bs{0})$. In both panels, the solid lines are the trajectories on the plane of $\big<\bs{x}_{t},{\bf c}\big>$ and $\big<\bs{y}_{t},{\bf c}\big>$. The black star markers are the Nash equilibria which satisfy $\big<\bs{x}_{*},{\bf c}\big>=\big<\bs{y}_{*},{\bf c}\big>=0$. In {\bf A}, the extra-optimistic weight is set to the minimum necessary value, $n=1$. In {\bf B}, it is to $n=m/2+1=6$. The step size is respectively fine-tuned as $\eta=10^{-1.95}$ in {\bf A} and $\eta=10^{-1.66}$ in {\bf B}. The black cross marks indicate the first $20$ steps of learning, which visualize that learning proceeds more quickly in {\bf B} than in {\bf A}. The trajectory also requires fewer cycles until convergence in {\bf B} than in {\bf A}.
    }
    \label{F01}
\end{figure*}

\begin{figure*}[h!]
    \centering
    \includegraphics[width=0.65\hsize]{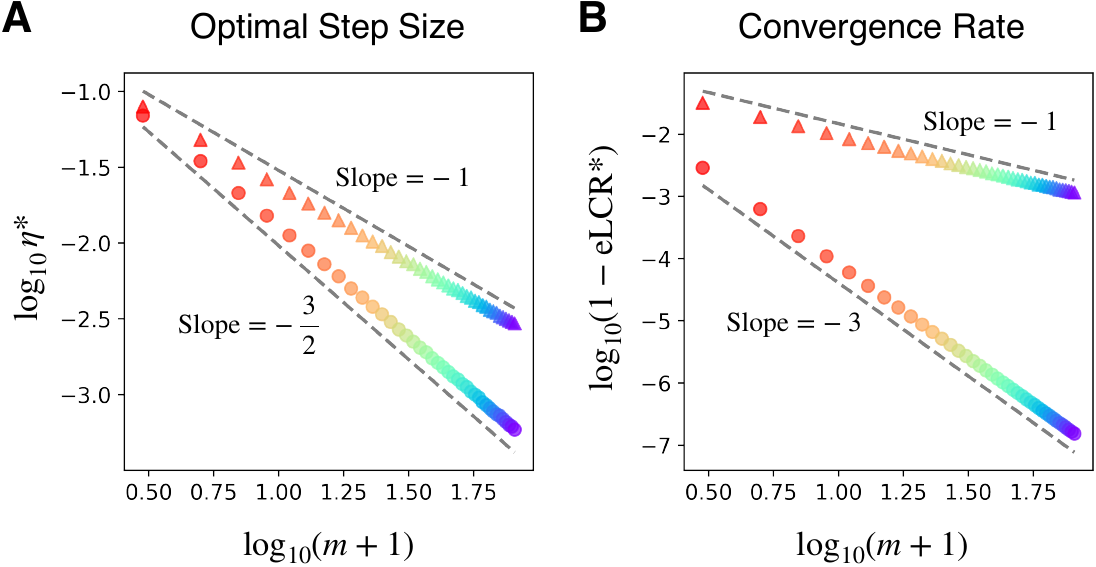}
    \caption{Optimal step size ({\bf A}) and optimal linear convergence rate ({\bf B}) in Matching Pennies. The circle and triangle markers indicate $n=1$ (next-step prediction) and $n=m/2+1$ (extra prediction), respectively. We consider the delays from $m=2$ (red) to $m=80$ (purple). In {\bf A}, the horizontal and vertical axes indicate logarithmic delay $\log_{10}(m+1)$ and logarithmic optimal step size $\log_{10}\eta^{*}$. The gray broken lines fit the circle markers with slope $-3/2$ and the triangle markers with slope $-1$, respectively. This means that the extra prediction admits a larger step size than the next-step prediction. In {\bf B}, the vertical axis indicates logarithmic linear convergence rate $\log_{10}(1-\rm{eLCR}^{*})$. The gray broken lines fit the circle markers with slope $-3$ and the triangle markers with slope $-1$, respectively. This means that the extra prediction converges faster than the next-step prediction.
    }
    \label{F02}
\end{figure*}

\begin{figure*}[h!]
    \centering
    \includegraphics[width=0.65\hsize]{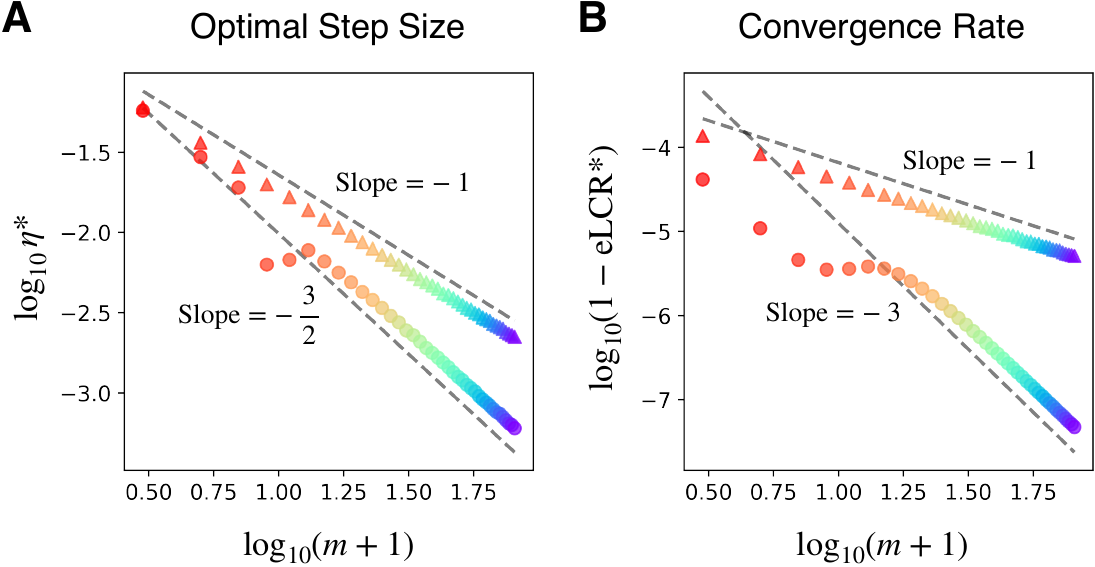}
    \caption{Optimal step size ({\bf A}) and optimal linear convergence rate ({\bf B}) in $5\times 5$ random matrix game. How to read the panels is the same as Figure \ref{F02}. The initial state is $(\bs{x}_{0},\bs{y}_{0})=(\bs{1},\bs{0})$.
    }
    \label{F03}
\end{figure*}

\subsection{Evaluation of Error} \label{ss_error}
The difference between EPP (i.e., $\bs{z}^{\dagger}_{t+1}$) and WOGDA (i.e., $\bs{z}_{t+1}$) is evaluated as follows.

\begin{lemma}[Dynamics of Error Term] \label{lem_dynamics_error}
For all $n\in\mathbb{N}$ and $0\le \eta$, it holds
\begin{align}
    \|\bs{z}^{\dagger}_{t+1}-\bs{z}_{t+1}\|^{2}&=\Delta^{2}\bs{\mc{Z}}_{t}^{\rm T}\frac{\eta^2\bar{\bs{A}}^{2}}{\bs{I}+n^{2}\eta^2\bar{\bs{A}}^{2}}\Delta^{2}\bs{\mc{Z}}_{t}, \label{analysis_error}\\
    \bs{\mc{Z}}_{t}&=n\sum_{s=1}^{m+1}\bs{z}_{t-s}+\sum_{s=2}^{m+1}\sum_{r=s}^{m+1}\bs{z}_{t-r}.
\end{align}
\end{lemma}

\begin{proof}
We can directly obtain
\begin{align}
    &\bs{z}^{\dagger}_{t+1}-\bs{z}_{t+1} \nonumber\\
    &=\{\eta\bs{A}\bs{z}_{t}+n\eta\bs{A}(\bs{z}^{\dagger}_{t+1}-\bs{z}_{t})\} \nonumber\\
    &\hspace{0.5cm}-\{\eta\bs{A}\bs{z}_{t-m}+(n+m)\eta\bs{A}\Delta\bs{z}_{t-m-1}\} \nonumber\\
    &=\eta\bs{A}(\bs{z}^{\dagger}_{t+1}-\bs{z}_{t+1})+\eta\bs{A}\{n\Delta\bs{z}_{t} \nonumber\\
    &\hspace{0.5cm}+(\Delta\bs{z}_{t-1}+\cdots+\Delta\bs{z}_{t-m})-(n+m)\Delta\bs{z}_{t-m-1}\} \label{analysis_error2}\\
    &=\eta\bs{A}(\bs{z}^{\dagger}_{t+1}-\bs{z}_{t+1})+\eta\bs{A}\Delta^{2}\bs{\mc{Z}}_{t}.
\end{align}
Here, we remark that in Eq.~\eqref{analysis_error2}, the numbers of $+\Delta\bs{z}$ and $-\Delta\bs{z}$ are balanced, it degenerates to a finite number of $\Delta^{2}\bs{z}$. By using the fraction notation of matrices, we can write
\begin{align}
    \bs{z}^{\dagger}_{t+1}-\bs{z}_{t+1}=\frac{\eta\bs{A}}{\bs{I}-n\eta\bs{A}}\Delta^{2}\bs{\mc{Z}}_{t},
\end{align}
leading to Eq.~\eqref{analysis_error} by performing the same operation as the proof of Lemma \ref{lem_dynamics_epp}.
\end{proof}

Here, we note that $\|\bs{z}^{\dagger}_{t+1}-\bs{z}_{t+1}\|$ is described as the finite number of $\Delta^{2}\bs{z}_{s}$, which is evaluated as sufficiently small below.

\begin{lemma}[Evaluation of Error Term] \label{lem_evaluation_error}
For all $j,n\in\mathbb{N}$ and $0\le\eta\le 1/(2(n+m)\lambda_{\max})$, it holds
\begin{align}
    &\|\bs{z}^{\dagger}_{t+1}-\bs{z}_{t+1}\|\le\rm{ER}(j,n)\max_{2m\le s\le 2j(m+1)}\|\bs{z}_{t-s}\|, \label{ErrorRateBound}\\
    &\rm{ER}(j,n)=2(2n+m)(m+1)\eta^{3}\lambda_{\max}^{3} \nonumber\\
    &\hspace{1.9cm}+8(n+m)^{j+1}\eta^{j+2}\lambda_{\max}^{j+2} \nonumber\\
    &\hspace{1.9cm}+2(2n+m+1)(n+m)^{2j}\eta^{2j+1}\lambda_{\max}^{2j+1}. \tag{Error Rate}
\end{align}
\end{lemma}

{\it Proof Sketch.} By recursively substituting Eq.~\eqref{WOGDA_delta} $j$ times, we obtain
\begin{align}
    \Delta\bs{z}_{t}=&\ \sum_{k=1}^{j}(n+m)^{k-1}\eta^{k}\bs{A}^{k}\bs{z}_{t-k(m+1)+1} \nonumber\\
    &\ +(n+m)^{j}\eta^{j}\bs{A}^{j}\Delta\bs{z}_{t-j(m+1)}. \label{recursive}
\end{align}
$\rm{ER}(j,n)$ is obtained by applying Eq.~\eqref{recursive} to $\Delta^{2}\bs{\mc{Z}}_{t}$ twice and upper-bounding all $\bs{z}_{t-s}$ by $\max_{s}\bs{z}_{t-s}$. Here, the first term of $\rm{ER}(j,n)$ corresponds to the product of the first terms of Eq.~\eqref{recursive}. The second term of $\rm{ER}(j,n)$ corresponds to the product of the first and second terms of Eq.~\eqref{recursive}. The third term of $\rm{ER}(j,n)$ corresponds to the product of the second terms of Eq.~\eqref{recursive}.
\qed

\subsection{Combination of EPP and Error} \label{ss_combination}
By combining the convergence term by EPP and the error term, we can analyze WOGDA as follows.

\begin{theorem}[Analysis of WOGDA] \label{thm_convergence_wogda}
For all $j,n\in\mathbb{N}$ and $0\le\eta\le 1/(2(n+m)\lambda_{\max})$, we derive
\begin{align}
    &\|\bs{z}_{t+1}\|\le\rm{LCR}_{\rm{WOGDA}}(j,n)\max_{0\le s\le 2j(m+1)}\|\bs{z}_{t-s}\|, \label{LCR_WOGDA}\\
    &\rm{LCR}_{\rm{WOGDA}}(j,n)=\rm{LCR}_{\rm{EPP}}(n)+\rm{ER}(j,n). \tag{Linear Convergence Rate of WOGDA}
\end{align}
\end{theorem}

\begin{proof}
Eq.~\eqref{LCR_WOGDA} is obvious by the triangle inequality of
\begin{align}
    &\|\bs{z}_{t+1}\|\le \|\bs{z}^{\dagger}_{t+1}\|+\|\bs{z}^{\dagger}_{t+1}-\bs{z}_{t+1}\|,
\end{align}
with the results from Lemma \ref{lem_analysis_epp} and Lemma \ref{lem_evaluation_error}
\begin{align}
    \|\bs{z}^{\dagger}_{t+1}\|&\le \rm{LCR}_{\rm{EPP}}(n)\|\bs{z}_{t}\|, \\
    \|\bs{z}^{\dagger}_{t+1}-\bs{z}_{t+1}\|&\le \rm{ER}(j,n)\max_{2m\le s\le 2j(m+1)}\|\bs{z}_{t-s}\|.
\end{align}
\end{proof}

Finally, Theorem \ref{thm_convergence_wogda} immediately derives Theorems \ref{thm_convergence_pp} and \ref{thm_convergence_epp} by setting the step size properly to cancel out the convergence and error terms. The proof sketches are below.

{\it Proof Sketch of} {\bf Theorem \ref{thm_convergence_pp}}. When $n=1$, WOGDA linearly converges with the rate of $\rm{LCR}_{\rm{EPP}}(1)=1-O(\eta^{2}\lambda_{\min}^{2})$. On the other hand, it also diverges with the rate of $\rm{ER}(j,1)=O(m^{2}\eta^{3}\lambda_{\max}^{3})$. In order to cancels out the terms of $O(\eta^{2}\lambda_{\min}^{2})$ with $O(m^{2}\eta^{3}\lambda_{\max}^{3})$, we should set the step size as $\eta=O(1/(m^{2}\kappa^{2}\lambda_{\max}))$. When $j\ge 2$, the other terms can be ignored with this step size. Therefore, we set $j=2$ at best and obtain the convergence rate as $\rm{LCR}_{\rm{WOGDA}}(2,1)=1-O(\eta^{2}\lambda_{\min}^{2})=\exp(-\Theta(1/(\kappa^{6}m^{4})))$. This convergence rate scales down to $\exp(-\Theta(1/(\kappa^{6}m^{5})))$ because the convergence delays $O(m)$ times by the definition of Eq.~\eqref{LCR_WOGDA}.
\qed

{\it Proof Sketch of} {\bf Theorem \ref{thm_convergence_epp}}. When $n=m/2+1$, WOGDA linearly converges with the rate of $\rm{LCR}_{\rm{EPP}}(m/2+1)=1-O(m\eta^{2}\lambda_{\min}^{2})$. On the other hand, it also diverges by the effect of $\rm{ER}(j,n)$, and the leading term is $O(m^{2j+1}\eta^{2j+1}\lambda_{\max}^{2j+1})$. In order to cancels out the terms of $O(m\eta^{2}\lambda_{\min}^{2})$ with $O(m^{2j+1}\eta^{2j+1}\lambda_{\max}^{2j+1})$, we should set the step size as $\eta=O(1/(m^{\frac{2j}{2j-1}}\kappa^{2}\lambda_{\max}))$, immediately leading to the convergence rate of $\rm{LCR}_{\rm{WOGDA}}(j,m/2+1)=1-O(m\eta^{2}\lambda_{\min}^{2})=\exp(-\Theta(1/(\kappa^{6}m^{1+\frac{2}{2j-1}})))$. This rate scales down to $\exp(-\Theta(1/(\kappa^{6}jm^{2+\frac{2}{2j-1}})))$ because the convergence delays $O(jm)$ times by the definition of Eq.~\eqref{LCR_WOGDA}. It takes the minimum value of $\exp(-\Theta(1/(\kappa^{6}m^{2}\log m)))$ when $j=O(\log m)$.
\qed

\section{Experiments}
This section numerically tests the performance of the WOGDA, based on the two measures of the optimal step size $\eta^{*}$ and the estimated linear convergence rate there $\rm{eLCR}^{*}$. As far as we have conducted, our experiments guarantee that if WOGDA converges with a step size, it also converges with smaller step sizes. Thus, a larger step size contributes not only to faster convergence but also to the stability of convergence.

\paragraph{Experimental Setup:} For each delay $m$ and prediction length $n$, we run WOGDA with $251$ step sizes of $\eta=10^{-1.00}, 10^{-1.01},\cdots, 10^{-3.50}$. We stop the algorithm either $10^{-9}\le\|\bs{z}\|_{t}\le 10^{9}$ or $10^{4}\le t$ no longer holds. We measure the estimated linear convergence rate $\rm{eLCR}$ by the final $100$ steps of $\bs{z}_{t}$. Finally, the optimal step size $\eta^{*}$ is one such that minimize $\rm{eLCR}$, and the optimal estimated linear convergence rate $\rm{eLCR}^{*}$ is the minimum $\rm{eLCR}$.

First, let us pick up Matching Pennies, whose payoff matrix is given by
\begin{align}
    \bs{B}={\bf c}\otimes {\bf c},\quad {\bf c}:=\begin{pmatrix}
        +1 \\
        -1 \\
    \end{pmatrix}. \tag{Matching Pennies} \label{matching_pennies}
\end{align}
Because this payoff matrix is not regular, the Nash equilibrium is not unique but all $(\bs{x}_{*},\bs{y}_{*})$ such that $\big<\bs{x}_{*},{\bf c}\big>=\big<\bs{y}_{*},{\bf c}\big>=0$. We should employ another measure for the distance from these equilibria as $\sqrt{\big<\bs{x}_{t},{\bf c}\big>^{2}+\big<\bs{y}_{t},{\bf c}\big>^{2}}$. Matching Pennies is out of our theoretical scope, but our algorithm remains applicable.

\paragraph{Fast convergence by Extra Prediction:} We consider \ref{matching_pennies} with the delay $m=10$. Figure \ref{F01}-A considers the minimum optimistic weight $n=1$ (next-step prediction), while B considers the extra weight $n=m/2+1=6$ (extra prediction). In both A and B, we tuned the step size to the optimal one. We see two merits to take the extra weight. One is that the optimal step size is larger in B than in A. Another is that convergence to the Nash equilibrium is quick without cycling too many times.


\paragraph{Scaling law:} Keeping using Matching Pennies, we next see how the optimal step size and the linear convergence rate depend on the delay $m$. In Figure \ref{F02}-A, which compares the optimal step size $\eta^{*}$ between $n=1$ (next-step prediction) and $n=m/2+1$ (extra prediction), we see that it shrinks with delay as $\eta^{*}=O(1/m^{3/2})$ in $n=1$, but this shrinking becomes slower in $n=m/2+1$ as $\eta^{*}=O(1/m)$. On the other hand, Figure \ref{F02} compares the optimal linear convergence rate $\rm{eLCR}^{*}$, where we observe that convergence slows down with delay in $n=1$ as $\rm{eLCR}^{*}=1-O(1/m^{3})$, but it is faster in $n=m/2+1$ as $\rm{eLCR}^{*}=1-O(1/m)$.


\paragraph{Random-matrix game:} Not only for Matching Pennies, but we also conduct experiments for a randomly-generated $5\times 5$ matrix game, all of whose elements follows Gaussian distribution $\mc{N}(0,1)$. Figure \ref{F03} observes that the same scaling law as in Matching Pennies holds. An interesting observation is that the scaling law for $n=1$ is disordered a little as if the convergence speed must be lower than in $n=m/2+1$. We see similar scaling law in all $10$ random matrices generated independently.


\section{Discussion}

\paragraph{Comparison between theory and experiment:} Our theoretical and experimental results share the qualitative consensus that extra prediction permits larger step sizes and accelerates convergence. We remark that a quantitative gap remains between these theoretical and experimental results. Regarding the next-step prediction, the optimal step size is $O(1/m^{2})$ in theory versus $O(1/m^{3/2})$ in experiments, and the corresponding convergence rate is $\exp(-\Theta(t/m^{5}))$ in theory versus $\exp(-\Theta(t/m^{3}))$ in experiments. Similarly, for the extra prediction, the optimal step size is $O(1/(m\log m))$ in theory versus $O(1/m)$ in experiment, and the corresponding convergence rate is $\exp(-\Theta(t/(m^{2}\log m)))$ in theory versus $\exp(-\Theta(t/m))$ in experiment. Bridging the gap between these theories and experiments is an open problem. Especially for the extra prediction, however, the rate in theory could potentially be tightened to $\exp(-\Theta(t/(m\log m)))$ by discovering the Lyapunov function and avoiding the bound using the max function in Eq.~\eqref{ErrorRateBound}. 

\paragraph{Alternative algorithm that is memory complex:}
There is a way to force the existing algorithm \ref{OGDA} to adapt to delayed feedback. Suppose $m+1$ \ref{OGDA} algorithms, labeled as $\rm{OGDA}_{0}$ to $\rm{OGDA}_{m}$, in parallel, and the reward at time $t$ is assigned to $\rm{OGDA}_{t\mod (m+1)}$. Although this parallelized algorithm trivially achieves the linear convergence rate of $\exp(-\Theta(t/m))$ by applying the previous paper~\cite{mokhtari2020unified}, it crucially suffers from the memory complexity of $O(m)$. Because this memory complexity is obviously incomparable to ours $O(1)$, we did not compare the parallelized algorithm with ours, but we note that at least, WOGDA also achieves the rate of $\exp(-\Theta(t/m))$ as well as the memory complexity of $O(1)$ in our experiments.

\paragraph{Robustness to stochastic and unknown delays:} Although our algorithm assumes a fixed and known delay $m\in\mathbb{N}$, we hypothesize that exact knowledge of $m$ is not strictly required to enjoy the qualitative benefits of extra prediction. In real-world scenarios such as online advertising, feedback delays are often stochastic or unknown, but their time scale is thought to be easily observable. As shown in the proof sketches in Theorems \ref{thm_convergence_pp} and \ref{thm_convergence_epp}, we have discussed how the term contributing to convergence (i.e., $\rm{LCR}_{\rm{EPP}}(n)$) cancels out that contributing to divergence (i.e., $\rm{ER}(j,n)$), and this discussion is applicable to uncertainty in delays that does not affect the scale. Although a precise step size under unknown delays remains an open problem, this observation suggests that WOGDA could be robust to such uncertain delays.

\paragraph{Conclusion:} 
This paper has established the rate of linear convergence in learning in games with feedback delays. The convergence was proven based on unconstrained bilinear games, complementary to LIC in past constrained settings. This linear rate is a strong contribution to agents quickly stabilizing their optimal strategies. Furthermore, we successfully accelerated this convergence by introducing a novel mechanism that predicts the future reward farther away than the next step. This study developed a strong countermeasure against delayed feedback in learning in games.

\section*{Impact Statement}
This paper presents work whose goal is to advance the field of Learning in Games, Multi-Agent Learning, and Online Learning. There are many potential societal consequences of our work, none which we feel must be specifically highlighted here.


\begin{thebibliography}{33}
\providecommand{\natexlab}[1]{#1}
\providecommand{\url}[1]{\texttt{#1}}
\expandafter\ifx\csname urlstyle\endcsname\relax
  \providecommand{\doi}[1]{doi: #1}\else
  \providecommand{\doi}{doi: \begingroup \urlstyle{rm}\Url}\fi

\bibitem[Agarwal \& Duchi(2011)Agarwal and Duchi]{agarwal2011distributed}
Agarwal, A. and Duchi, J.~C.
\newblock Distributed delayed stochastic optimization.
\newblock In \emph{NeurIPS}, 2011.

\bibitem[Cesa-Bianchi et~al.(2016)Cesa-Bianchi, Gentile, Mansour, and
  Minora]{cesa2016delay}
Cesa-Bianchi, N., Gentile, C., Mansour, Y., and Minora, A.
\newblock Delay and cooperation in nonstochastic bandits.
\newblock In \emph{COLT}, 2016.

\bibitem[Cesa-Bianchi et~al.(2018)Cesa-Bianchi, Gentile, and
  Mansour]{cesa2018nonstochastic}
Cesa-Bianchi, N., Gentile, C., and Mansour, Y.
\newblock Nonstochastic bandits with composite anonymous feedback.
\newblock In \emph{COLT}, 2018.

\bibitem[Cevher et~al.(2023)Cevher, Piliouras, Sim, and
  Skoulakis]{cevher2023min}
Cevher, V., Piliouras, G., Sim, R., and Skoulakis, S.
\newblock Min-max optimization made simple: Approximating the proximal point
  method via contraction maps.
\newblock In \emph{SOSA}, 2023.

\bibitem[Chapelle(2014)]{chapelle2014modeling}
Chapelle, O.
\newblock Modeling delayed feedback in display advertising.
\newblock In \emph{KDD}, 2014.

\bibitem[Daskalakis et~al.(2018)Daskalakis, Ilyas, Syrgkanis, and
  Zeng]{daskalakis2018training}
Daskalakis, C., Ilyas, A., Syrgkanis, V., and Zeng, H.
\newblock Training gans with optimism.
\newblock In \emph{ICLR}, 2018.

\bibitem[Desautels et~al.(2014)Desautels, Krause, and
  Burdick]{desautels2014parallelizing}
Desautels, T., Krause, A., and Burdick, J.~W.
\newblock Parallelizing exploration-exploitation tradeoffs in gaussian process
  bandit optimization.
\newblock \emph{J. Mach. Learn. Res.}, 15\penalty0 (1):\penalty0 3873--3923,
  2014.

\bibitem[Feng et~al.(2023)Feng, Fu, Hu, Li, Panageas, Wang,
  et~al.]{feng2023last}
Feng, Y., Fu, H., Hu, Q., Li, P., Panageas, I., Wang, X., et~al.
\newblock On the last-iterate convergence in time-varying zero-sum games: Extra
  gradient succeeds where optimism fails.
\newblock In \emph{NeurIPS}, 2023.

\bibitem[Feng et~al.(2024)Feng, Li, Panageas, and Wang]{feng2024last}
Feng, Y., Li, P., Panageas, I., and Wang, X.
\newblock Last-iterate convergence separation between extra-gradient and
  optimism in constrained periodic games.
\newblock In \emph{UAI}, 2024.

\bibitem[Fujimoto et~al.(2025{\natexlab{a}})Fujimoto, Abe, and
  Ariu]{fujimoto2025learning}
Fujimoto, Y., Abe, K., and Ariu, K.
\newblock Learning from delayed feedback in games via extra prediction.
\newblock In \emph{NeurIPS}, 2025{\natexlab{a}}.

\bibitem[Fujimoto et~al.(2025{\natexlab{b}})Fujimoto, Ariu, and
  Abe]{fujimoto2025synchronization}
Fujimoto, Y., Ariu, K., and Abe, K.
\newblock Synchronization in learning in periodic zero-sum games triggers
  divergence from nash equilibrium.
\newblock In \emph{AAAI}, 2025{\natexlab{b}}.

\bibitem[Hsieh et~al.(2022)Hsieh, Iutzeler, Malick, and
  Mertikopoulos]{hsieh2022multi}
Hsieh, Y.-G., Iutzeler, F., Malick, J., and Mertikopoulos, P.
\newblock Multi-agent online optimization with delays: Asynchronicity,
  adaptivity, and optimism.
\newblock \emph{Journal of Machine Learning Research}, 23\penalty0
  (78):\penalty0 1--49, 2022.

\bibitem[Joulani et~al.(2013)Joulani, Gyorgy, and
  Szepesv{\'a}ri]{joulani2013online}
Joulani, P., Gyorgy, A., and Szepesv{\'a}ri, C.
\newblock Online learning under delayed feedback.
\newblock In \emph{ICML}, 2013.

\bibitem[Joulani et~al.(2016)Joulani, Gyorgy, and
  Szepesv{\'a}ri]{joulani2016delay}
Joulani, P., Gyorgy, A., and Szepesv{\'a}ri, C.
\newblock Delay-tolerant online convex optimization: Unified analysis and
  adaptive-gradient algorithms.
\newblock In \emph{AAAI}, 2016.

\bibitem[Li et~al.(2019)Li, Chen, and Giannakis]{li2019bandit}
Li, B., Chen, T., and Giannakis, G.~B.
\newblock Bandit online learning with unknown delays.
\newblock In \emph{AISTATS}, 2019.

\bibitem[McMahan \& Streeter(2014)McMahan and Streeter]{mcmahan2014delay}
McMahan, B. and Streeter, M.
\newblock Delay-tolerant algorithms for asynchronous distributed online
  learning.
\newblock In \emph{NeurIPS}, 2014.

\bibitem[Mertikopoulos et~al.(2019)Mertikopoulos, Lecouat, Zenati, Foo,
  Chandrasekhar, and Piliouras]{mertikopoulos2019optimistic}
Mertikopoulos, P., Lecouat, B., Zenati, H., Foo, C.-S., Chandrasekhar, V., and
  Piliouras, G.
\newblock Optimistic mirror descent in saddle-point problems: Going the extra
  (gradient) mile.
\newblock In \emph{ICLR}, 2019.

\bibitem[Mokhtari et~al.(2020)Mokhtari, Ozdaglar, and
  Pattathil]{mokhtari2020unified}
Mokhtari, A., Ozdaglar, A., and Pattathil, S.
\newblock A unified analysis of extra-gradient and optimistic gradient methods
  for saddle point problems: Proximal point approach.
\newblock In \emph{AISTATS}, 2020.

\bibitem[Neu et~al.(2010)Neu, Antos, Gy{\"o}rgy, and
  Szepesv{\'a}ri]{neu2010online}
Neu, G., Antos, A., Gy{\"o}rgy, A., and Szepesv{\'a}ri, C.
\newblock Online markov decision processes under bandit feedback.
\newblock In \emph{NeurIPS}, 2010.

\bibitem[Pike-Burke et~al.(2018)Pike-Burke, Agrawal, Szepesvari, and
  Grunewalder]{pike2018bandits}
Pike-Burke, C., Agrawal, S., Szepesvari, C., and Grunewalder, S.
\newblock Bandits with delayed, aggregated anonymous feedback.
\newblock In \emph{ICML}, 2018.

\bibitem[Piliouras et~al.(2022)Piliouras, Sim, and
  Skoulakis]{piliouras2022beyond}
Piliouras, G., Sim, R., and Skoulakis, S.
\newblock Beyond time-average convergence: Near-optimal uncoupled online
  learning via clairvoyant multiplicative weights update.
\newblock In \emph{NeurIPS}, 2022.

\bibitem[Quanrud \& Khashabi(2015)Quanrud and Khashabi]{quanrud2015online}
Quanrud, K. and Khashabi, D.
\newblock Online learning with adversarial delays.
\newblock In \emph{NeurIPS}, 2015.

\bibitem[Rakhlin \& Sridharan(2013)Rakhlin and
  Sridharan]{rakhlin2013optimization}
Rakhlin, S. and Sridharan, K.
\newblock Optimization, learning, and games with predictable sequences.
\newblock In \emph{NeurIPS}, 2013.

\bibitem[Shamir \& Szlak(2017)Shamir and Szlak]{shamir2017online}
Shamir, O. and Szlak, L.
\newblock Online learning with local permutations and delayed feedback.
\newblock In \emph{ICML}, 2017.

\bibitem[Syrgkanis et~al.(2015)Syrgkanis, Agarwal, Luo, and
  Schapire]{syrgkanis2015fast}
Syrgkanis, V., Agarwal, A., Luo, H., and Schapire, R.~E.
\newblock Fast convergence of regularized learning in games.
\newblock In \emph{NeurIPS}, 2015.

\bibitem[Vernade et~al.(2017)Vernade, Capp{\'e}, and
  Perchet]{vernade2017stochastic}
Vernade, C., Capp{\'e}, O., and Perchet, V.
\newblock Stochastic bandit models for delayed conversions.
\newblock In \emph{UAI}, 2017.

\bibitem[Wei et~al.(2021)Wei, Lee, Zhang, and Luo]{wei2021linear}
Wei, C.-Y., Lee, C.-W., Zhang, M., and Luo, H.
\newblock Linear last-iterate convergence in constrained saddle-point
  optimization.
\newblock In \emph{ICLR}, 2021.

\bibitem[Weinberger \& Ordentlich(2002)Weinberger and
  Ordentlich]{weinberger2002delayed}
Weinberger, M.~J. and Ordentlich, E.
\newblock On delayed prediction of individual sequences.
\newblock \emph{IEEE Transactions on Information Theory}, 48\penalty0
  (7):\penalty0 1959--1976, 2002.

\bibitem[Yasui et~al.(2020)Yasui, Morishita, Komei, and
  Shibata]{yasui2020feedback}
Yasui, S., Morishita, G., Komei, F., and Shibata, M.
\newblock A feedback shift correction in predicting conversion rates under
  delayed feedback.
\newblock In \emph{WWW}, 2020.

\bibitem[Yoshikawa \& Imai(2018)Yoshikawa and Imai]{yoshikawa2018nonparametric}
Yoshikawa, Y. and Imai, Y.
\newblock A nonparametric delayed feedback model for conversion rate
  prediction.
\newblock \emph{arXiv preprint arXiv:1802.00255}, 2018.

\bibitem[Zheng et~al.(2017)Zheng, Meng, Wang, Chen, Yu, Ma, and
  Liu]{zheng2017asynchronous}
Zheng, S., Meng, Q., Wang, T., Chen, W., Yu, N., Ma, Z.-M., and Liu, T.-Y.
\newblock Asynchronous stochastic gradient descent with delay compensation.
\newblock In \emph{ICML}, 2017.

\bibitem[Zhou et~al.(2017)Zhou, Mertikopoulos, Bambos, Glynn, and
  Tomlin]{zhou2017countering}
Zhou, Z., Mertikopoulos, P., Bambos, N., Glynn, P.~W., and Tomlin, C.
\newblock Countering feedback delays in multi-agent learning.
\newblock In \emph{NeurIPS}, 2017.

\bibitem[Zinkevich et~al.(2009)Zinkevich, Langford, and
  Smola]{zinkevich2009slow}
Zinkevich, M., Langford, J., and Smola, A.
\newblock Slow learners are fast.
\newblock In \emph{NeurIPS}, 2009.

\end{thebibliography}

\newpage
\appendix
\onecolumn

\renewcommand{\theequation}{A\arabic{equation}}
\setcounter{equation}{0}
\renewcommand{\figurename}{Figure A}
\setcounter{figure}{0}

\begin{center}
{\LARGE\bf Appendix}
\end{center}

\section{Flowchart of the lemmas and theorems}

\begin{figure*}[h!]
    \centering
    \includegraphics[width=0.8\hsize]{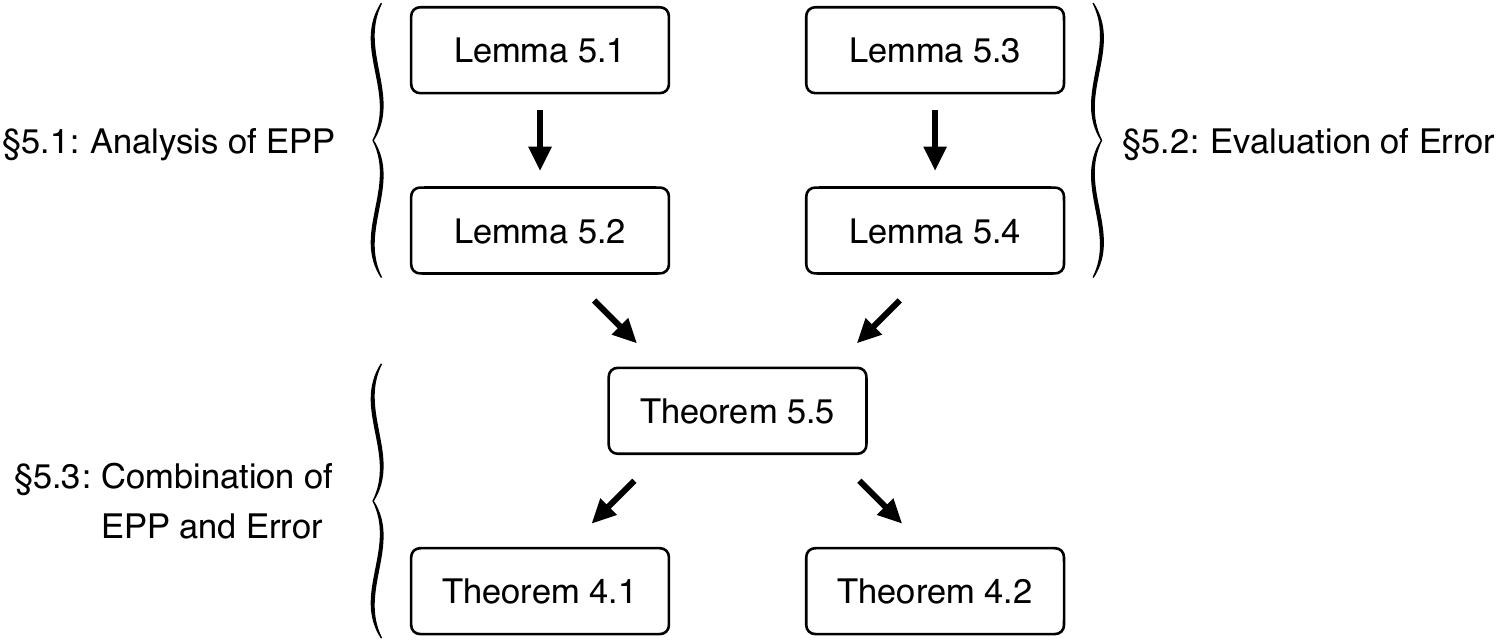}
    \caption{Flowchart of the lemmas and theorems. In \S\ref{ss_epp}, Lemma \ref{lem_dynamics_epp} captures the dynamics of EPP, and \ref{lem_analysis_epp} analyzes them. In \S\ref{ss_error}, Lemma \ref{lem_dynamics_error} captures the dynamics of the error between WOGDA and EPP, and Lemma \ref{lem_evaluation_error} analyzes them. In \S\ref{ss_combination}, Theorem \ref{thm_convergence_wogda} analyzes WOGDA by combining the convergence term by EPP and the error term. It leads to Theorem \ref{thm_convergence_pp}, providing the rate of linear convergence by next-step prediction, and Theorem \ref{thm_convergence_epp}, providing the rate of linear convergence by extra prediction.
    }
    \label{FA01}
\end{figure*}

\section{Proof of Lemma \ref{lem_dynamics_epp}}
\begin{proof}
The dynamics of \ref{EPP} are rewritten as
\begin{align}
    (\bs{I}-n\eta\bs{A})\bs{z}^{\dagger}_{t+1}=(\bs{I}-(n-1)\eta\bs{A})\bs{z}_{t}.
\end{align}
Since $\bs{A}$ is skew-symmetric matrix ($\bs{A}^{\rm T}=-\bs{A}$), $\bs{I}+\alpha\bs{A}$ has its inverse matrix for all $\alpha\in\mathbb{R}$. Furthermore, for all $\alpha,\alpha'\in\mathbb{R}$, $(\bs{I}+\alpha\bs{A})^{-1}$ and $(\bs{I}+\alpha'\bs{A})$ are commutative. Therefore, we obtain
\begin{align}
    \bs{z}^{\dagger}_{t+1}=(\bs{I}-n\eta\bs{A})^{-1}(\bs{I}-(n-1)\eta\bs{A})\bs{z}_{t}=(\bs{I}-(n-1)\eta\bs{A})(\bs{I}-n\eta\bs{A})^{-1}\bs{z}_{t}.
\end{align}

Since the inverse matrix of the denominator exists and it can be commutative with the numerator, it is acceptable to abuse the following notation.
\begin{align}
    \bs{z}^{\dagger}_{t+1}=\frac{\bs{I}-(n-1)\eta\bs{A}}{\bs{I}-n\eta\bs{A}}\bs{z}_{t}. \label{z_dagger}
\end{align}

Finally, we obtain
\begin{align}
    \|\bs{z}^{\dagger}_{t+1}\|^{2}&\overset{\rm a}{=}\bs{z}_{t}^{\rm T}\frac{\bs{I}^{\rm T}-(n-1)\eta\bs{A}^{\rm T}}{\bs{I}^{\rm T}-n\eta\bs{A}^{\rm T}}\frac{\bs{I}-(n-1)\eta\bs{A}}{\bs{I}-n\eta\bs{A}}\bs{z}_{t} \\
    &\overset{\rm b}{=}\bs{z}_{t}^{\rm T}\frac{\bs{I}+(n-1)\eta\bs{A}}{\bs{I}+n\eta\bs{A}}\frac{\bs{I}-(n-1)\eta\bs{A}}{\bs{I}-n\eta\bs{A}}\bs{z}_{t} \\
    &=\bs{z}_{t}^{\rm T}\frac{\bs{I}-(n-1)^{2}\eta^{2}\bs{A}^{2}}{\bs{I}-n^{2}\eta^{2}\bs{A}^{2}}\bs{z}_{t} \\
    &\overset{\rm c}{=}\bs{z}_{t}^{\rm T}\frac{\bs{I}+(n-1)^{2}\eta^{2}\bar{\bs{A}}^{2}}{\bs{I}+n^{2}\eta^{2}\bar{\bs{A}}^{2}}\bs{z}_{t}.
\end{align}
In (a), we used the transpose of Eq.~\eqref{z_dagger}. In (b), we used $\bs{I}^{\rm T}=\bs{I}$ and $\bs{A}^{\rm T}=-\bs{A}$. In (c), we used $\bs{A}^{2}=-\bar{\bs{A}}^{2}$.
\end{proof}

\section{Proof of Lemma \ref{lem_evaluation_error}}
By substituting Eq.~\eqref{WOGDA_delta} recursively $j\in\mathbb{N}$ times, we obtain
\begin{align}
    \Delta\bs{z}_{t}&=\eta\bs{A}\bs{z}_{t-m}+(n+m)\eta\bs{A}\Delta\bs{z}_{t-m-1} \\
    &=\eta\bs{A}\bs{z}_{t-m}+(n+m)\eta\bs{A}(\eta\bs{A}\bs{z}_{t-2m-1}+(n+m)\eta\bs{A}\Delta\bs{z}_{t-2m-2}) \\
    &=\cdots=\sum_{k=1}^{j}(n+m)^{k-1}\eta^{k}\bs{A}^{k}\bs{z}_{t-k(m+1)+1}+(n+m)^{j}\eta^{j}\bs{A}^{j}\Delta\bs{z}_{t-j(m+1)}. \label{deltaz}
\end{align}

Also, $\Delta^{2}\bs{z}_{t}$ is evaluated as
\begin{align}
    \Delta^{2}\bs{z}_{t}\overset{\rm a}{=}&\ \sum_{k=1}^{j}(n+m)^{k-1}\eta^{k}\bs{A}^{k}\Delta\bs{z}_{t-k(m+1)+1}+(n+m)^{j}\eta^{j}\bs{A}^{j}\Delta^{2}\bs{z}_{t-j(m+1)} \\
    \overset{\rm b}{=}&\ \sum_{k=1}^{j}(n+m)^{k-1}\eta^{k}\bs{A}^{k}\left(\sum_{l=1}^{j}(n+m)^{l-1}\eta^{l}\bs{A}^{l}\bs{z}_{t-(k+l)(m+1)+2}+(n+m)^{j}\eta^{j}\bs{A}^{j}\Delta\bs{z}_{t-(k+j)(m+1)+1}\right) \\
    &+(n+m)^{j}\eta^{j}\bs{A}^{j}\left(\sum_{k=1}^{j}(n+m)^{k-1}\eta^{k}\bs{A}^{k}\Delta\bs{z}_{t-(j+k)(m+1)+1}+(n+m)^{j}\eta^{j}\bs{A}^{j}\Delta^{2}\bs{z}_{t-2j(m+1)}\right) \\
    =&\ \sum_{k=1}^{j}\sum_{l=1}^{j}(n+m)^{k+l-2}\eta^{k+l}\bs{A}^{k+l}\bs{z}_{t-(k+l)(m+1)+2} \\
    &+2\sum_{k=1}^{j}(n+m)^{j+k}\eta^{j+k}\bs{A}^{j+k}\Delta\bs{z}_{t-(j+k)(m+1)+1}+(n+m)^{2j}\eta^{2j}\bs{A}^{2j}\Delta^{2}\bs{z}_{t-2j(m+1)}.
\end{align}
In (a) and (b), we applied Eq.~\eqref{deltaz} repeatedly.

Suppose sufficiently small step size $\eta\le 1/(2(n+m)\lambda_{\max})$, then we obtain
\begin{align}
    &\|\Delta^{2}\mc{Z}_{t}\| \\
    &\overset{\rm a}{\le}\Bigg(\left(n+\frac{m}{2}\right)(m+1)\sum_{k=1}^{j}\sum_{l=1}^{j}(n+m)^{k+l-2}\eta^{k+l}\lambda_{\max}^{k+l}+4(n+m)\sum_{k=1}^{j}(n+m)^{j+k-1}\eta^{j+k}\lambda_{\max}^{j+k}\Bigg. \nonumber\\
    &\hspace{0.6cm}\Bigg.+2(2n+m+1)(n+m)^{2j}\eta^{2j}\lambda_{\max}^{2j}\Bigg)\max_{2m\le s\le 2j(m+1)}\|\bs{z}_{t-s}\| \\
    &\overset{\rm b}{=}\Bigg(\left(n+\frac{m}{2}\right)(m+1)\eta^{2}\lambda_{\max}^{2}\left(\frac{1-(n+m)^{j}\eta^{j}\lambda_{\max}^{j}}{1-(n+m)\eta\lambda_{\max}}\right)^{2}+4(n+m)(n+m)^{j}\eta^{j+1}\lambda_{\max}^{j+1}\frac{1-(n+m)^{j}\eta^{j}\lambda_{\max}^{j}}{1-(n+m)\eta\lambda_{\max}}\Bigg. \nonumber\\
    &\hspace{0.6cm}\Bigg.+2(2n+m+1)(n+m)^{2j}\eta^{2j}\lambda_{\max}^{2j}\Bigg)\max_{2m\le s\le 2j(m+1)}\|\bs{z}_{t-s}\| \\
    &\overset{\rm c}{\le}\left(2(2n+m)(m+1)\eta^{2}\lambda_{\max}^{2}+8(n+m)^{j+1}\eta^{j+1}\lambda_{\max}^{j+1}+2(2n+m+1)(n+m)^{2j}\eta^{2j}\lambda_{\max}^{2j}\right)\max_{2m\le s\le 2j(m+1)}\|\bs{z}_{t-s}\|.
\end{align}
In (a), we upper-bounded $\bs{A}$ by $\lambda_{\max}$ and also upper-bounded $\bs{z}_{t}$, $\Delta\bs{z}_{t}$, and $\Delta^{2}\bs{z}_{t}$ in $\Delta^{2}\bs{\mc{Z}}_{t}$ as
\begin{align}
    \left\|n\sum_{s=1}^{m+1}\bs{z}_{t-s}+\sum_{s=2}^{m+1}\sum_{r=s}^{m+1}\bs{z}_{t-r}\right\|&\le n\sum_{s=1}^{m+1}\|\bs{z}_{t-s}\|+\sum_{s=2}^{m+1}\sum_{r=s}^{m+1}\|\bs{z}_{t-r}\| \\
    &\le \left(n+\frac{m}{2}\right)(m+1)\max_{1\le s\le m+1}\|\bs{z}_{t-s}\|, \\
    \left\|n\sum_{s=1}^{m+1}\Delta\bs{z}_{t-s}+\sum_{s=2}^{m+1}\sum_{r=s}^{m+1}\Delta\bs{z}_{t-r}\right\|&=\left\|n(\bs{z}_{t}-\bs{z}_{t-(m+1)})+\sum_{s=2}^{m+1}(\bs{z}_{t-s+1}-\bs{z}_{t-(m+1)})\right\| \\
    &\le n(\|\bs{z}_{t}\|+\|\bs{z}_{t-(m+1)}\|)+\sum_{s=2}^{m+1}(\|\bs{z}_{t-s+1}\|+\|\bs{z}_{t-(m+1)}\|) \\
    &\le 2(n+m)\max_{0\le s\le m+1}\|\bs{z}_{t-s}\|, \\
    \left\|n\sum_{s=1}^{m+1}\Delta^{2}\bs{z}_{t-s}+\sum_{s=2}^{m+1}\sum_{r=s}^{m+1}\Delta^{2}\bs{z}_{t-r}\right\|&=\left\|n(\bs{z}_{t+1}-\bs{z}_{t})+(\bs{z}_{t}-\bs{z}_{t-m})-(n+m)(\bs{z}_{t-m}-\bs{z}_{t-(m+1)})\right\| \\
    &\le n(\|\bs{z}_{t+1}\|+\|\bs{z}_{t}\|)+(\|\bs{z}_{t}\|+\|\bs{z}_{t-m}\|)+(n+m)(\|\bs{z}_{t-m}\|+\|\bs{z}_{t-(m+1)}\|) \\
    &\le 2(2n+m+1)\max_{-1\le s\le m+1}\|\bs{z}_{t-s}\|.
\end{align}
In (b), we summed up the geometric series. In (c), we used $\eta\le 1/(2(n+m)\lambda_{\max})$.

Finally, by Lemma \ref{lem_dynamics_error}
\begin{align}
    \|\bs{z}^{\dagger}_{t+1}-\bs{z}_{t+1}\|^{2}&=\Delta^{2}\bs{\mc{Z}}_{t}^{\rm T}\frac{\eta^2\bar{\bs{A}}^{2}}{\bs{I}+n^{2}\eta^2\bar{\bs{A}}^{2}}\Delta^{2}\bs{\mc{Z}}_{t} \\
    &\le \frac{\eta^{2}\lambda_{\max}^{2}}{1+n^{2}\eta^{2}\lambda_{\max}^{2}}\|\Delta^{2}\bs{\mc{Z}}_{t}\|^{2} \\
    &\le \eta^{2}\lambda_{\max}^{2}\|\Delta^{2}\bs{\mc{Z}}_{t}\|^{2} \\
    &\le \rm{ER}(j,n)^{2}\|\Delta^{2}\bs{\mc{Z}}_{t}\|^{2}.
\end{align}

\section{Proof of Theorem \ref{thm_convergence_pp}}
By substituting $n=1$ and $j=2$, we obtain linear convergence rate as
\begin{align}
    \rm{LCR}_{\rm{WOGDA}}(2,1)=&\ 1-\frac{1}{2}\eta^{2}\lambda_{\min}^{2}+\frac{1}{2}\eta^{4}\lambda_{\min}^{4}+2(m+2)(m+1)\eta^{3}\lambda_{\max}^{3} \\
    &+8(m+1)^{3}\eta^{4}\lambda_{\max}^{4}+2(m+3)(m+1)^{4}\eta^{5}\lambda_{\max}^{5}.
\end{align}

For $\eta=1/56(m+1)^{2}\kappa^{2}\lambda_{\max}$, this convergence rate is upper-bounded as
\begin{align}
    \rm{LCR}_{\rm{WOGDA}}(2,1)=&\ 1-\frac{1}{2}\eta^{2}\lambda_{\min}^{2}+\frac{1}{2}\eta^{4}\lambda_{\min}^{4}+2(m+2)(m+1)\eta^{3}\lambda_{\max}^{3} \\
    &+8(m+1)^{3}\eta^{4}\lambda_{\max}^{4}+2(m+3)(m+1)^{4}\eta^{5}\lambda_{\max}^{5} \\
    \le&\ 1-\frac{1}{2}\eta^{2}\lambda_{\min}^{2}+\frac{1}{56}\bigg\{\frac{1}{2}\eta^{2}\lambda_{\min}^{2}+4\eta^{2}\lambda_{\min}^{2}+8\eta^{2}\lambda_{\min}^{2}+6\eta^{2}\lambda_{\min}^{2}\bigg\} \\
    \le&\ 1-\frac{1}{6}\eta^{2}\lambda_{\min}^{2} \\
    =&\ 1-\frac{1}{6\cdot 56^{2}\kappa^{6}(m+1)^{4}}.
\end{align}

By the definition of Eq.~\eqref{LCR_WOGDA}, this immediately derives
\begin{align}
    \|\bs{z}_{t+1}\|&\le \left(1-\frac{1}{6\cdot 56^{2}\kappa^{6}(m+1)^{4}}\right)\max_{0\le s\le 4(m+1)}\|\bs{z}_{t-s}\| \\
    &\le \exp\left(-\frac{t}{6\cdot 56^{2}\kappa^{6}(m+1)^{4}}\right)\max_{0\le s\le 4(m+1)}\|\bs{z}_{t-s}\|.
\end{align}
By applying this $t$ times and defining $\|\tilde{\bs{z}}_{0}\|=\max_{0\le s\le 4(m+1)}\|\bs{z}_{s}\|$, we obtain
\begin{align}
    \|\bs{z}_{t}\|&\le \exp\left(-\frac{t}{6\cdot 56^{2}\kappa^{6}(m+1)^{4}}\times \frac{1}{4(m+1)}\right)\|\tilde{\bs{z}}_{0}\| \\
    &=\exp\left(-\frac{t}{4\cdot 6\cdot 56^{2}\kappa^{6}(m+1)^{5}}\right)\|\tilde{\bs{z}}_{0}\|.
\end{align}
This results in Eq.~\eqref{rate_pp} with positive constant $c=4\cdot 6\cdot 56^{2}$. 

\section{Proof of Theorem \ref{thm_convergence_epp}}
By substituting $n=m/2+1$, we obtain linear convergence rate as
\begin{align}
    \rm{LCR}_{\rm{WOGDA}}(j,\frac{m}{2}+1)=&\ 1-\frac{1}{2}(m+1)\eta^{2}\lambda_{\min}^{2}+\frac{1}{2}\left(\frac{m}{2}+1\right)^{2}(m+1)\eta^{4}\lambda_{\min}^{4} \\
    &+4(m+1)^{2}\eta^{3}\lambda_{\max}^{3}+8\left(\frac{3m}{2}+1\right)^{j+1}\eta^{j+2}\lambda_{\max}^{j+2}+2\left(\frac{m}{2}+3\right)\left(\frac{3m}{2}+1\right)^{2j}\eta^{2j+1}\lambda_{\max}^{2j+1}.
\end{align}

For $\eta=1/(93(m+1)^{\frac{2j}{2j-1}}\kappa^{2}\lambda_{\max})$, this convergence rate is upper-bounded as
\begin{align}
    \rm{LCR}_{\rm{WOGDA}}(j,\frac{m}{2}+1)=&\ 1-\frac{1}{2}(m+1)\eta^{2}\lambda_{\min}^{2}+\frac{1}{2}\left(\frac{m}{2}+1\right)^{2}(m+1)\eta^{4}\lambda_{\min}^{4}+4(m+1)^{2}\eta^{3}\lambda_{\max}^{3} \\
    &+8\left(\frac{3m}{2}+1\right)^{j+1}\eta^{j+2}\lambda_{\max}^{j+2}+2\left(2m+3\right)\left(\frac{3m}{2}+1\right)^{2j}\eta^{2j+1}\lambda_{\max}^{2j+1} \\
    \le&\ 1-\frac{1}{2}(m+1)\eta^{2}\lambda_{\min}^{2}+\frac{1}{62}\bigg\{\frac{1}{2}(m+1)\eta^{2}\lambda_{\min}^{2}+4(m+1)\eta^{2}\lambda_{\min}^{2} \\
    &+8\left(\frac{3m}{2}+1\right)\eta^{2}\lambda_{\min}^{2}+2\left(2m+3\right)\eta^{2}\lambda_{\min}^{2}\bigg\} \\
    \le&\ 1-\frac{1}{6}(m+1)\eta^{2}\lambda_{\min}^{2} \\
    =&\ 1-\frac{1}{6\cdot 93^{2}\kappa^{6}(m+1)^{1+\frac{2}{2j-1}}}.
\end{align}
By substituting $j=\lfloor\log(m+1)\rfloor+2$, we obtain
\begin{align}
    \rm{LCR}_{\rm{WOGDA}}(\lfloor\log(m+1)\rfloor+2,\frac{m}{2}+1)&\le 1-\frac{1}{6\cdot 93^{2}\kappa^{6}(m+1)^{1+\frac{2}{2\lfloor\log(m+1)\rfloor+3}}} \\
    &\le  1-\frac{1}{6\cdot 93^{2}\kappa^{6}(m+1)^{1+\frac{1}{\log(m+1)}}} \\
    &=1-\frac{1}{6\cdot 93^{2}e\kappa^{6}(m+1)}.
\end{align}

By the definition of Eq.~\eqref{LCR_WOGDA}, this immediately derives
\begin{align}
    \|\bs{z}_{t+1}\|\le \left(1-\frac{1}{6\cdot 93^{2}e\kappa^{6}(m+1)}\right)\max_{0\le s\le 2(\lfloor\log(m+1)\rfloor+2)(m+1)}\|\bs{z}_{t-s}\|.
\end{align}
By applying this $t$ times and defining $\|\tilde{\bs{z}}_{0}\|=\max_{0\le s\le 2(m+1)(\log (m+1)+2)}\|\bs{z}_{s}\|$, we obtain
\begin{align}
    \|\bs{z}_{t}\|&\le \exp\left(-\frac{t}{6\cdot 93^{2}e\kappa^{6}(m+1)}\times\frac{1}{2(\lfloor\log(m+1)\rfloor+2)(m+1)}\right)\|\tilde{\bs{z}}_{0}\| \\
    &\le \exp\left(-\frac{t}{6\cdot 93^{2}e\kappa^{6}(m+1)}\times\frac{1}{2(\log(m+1)+2)(m+1)}\right)\|\tilde{\bs{z}}_{0}\| \\
    &\le \exp\left(-\frac{t}{2\cdot 6\cdot 93^{2}e\kappa^{6}(m+1)^{2}(\log(m+1)+2)}\right)\|\tilde{\bs{z}}_{0}\|.
\end{align}
This results in Eq.~\eqref{rate_epp} with positive constant $c=2\cdot 6\cdot 93^{2}e$.


\end{document}